\documentclass{article}

\usepackage{PRIMEarxiv}

\usepackage[utf8]{inputenc} 
\usepackage[T1]{fontenc}    
\usepackage{hyperref}       
\usepackage{url}            
\usepackage{booktabs}       
\usepackage{amsfonts}       
\usepackage{nicefrac}       
\usepackage{microtype}      
\usepackage{lipsum}
\usepackage{fancyhdr}       
\usepackage{graphicx}       
\usepackage{amsmath}
\usepackage{cleveref}
\usepackage{amssymb}
\usepackage{amsthm}
\usepackage{thmtools}
\usepackage{thm-restate}
\usepackage{enumitem}
\usepackage{caption}
\usepackage{subcaption}
\usepackage{pgfplots}
\usepackage{xcolor}
\usepackage{wrapfig}
\pgfplotsset{compat=1.18}
\newtheorem{theorem}{Theorem}

\newtheorem{lemma}[theorem]{Lemma}
\graphicspath{{media/}}     

\title{Increasing LLM response trustworthiness using voting ensembles
}

\author{
  Aparna Nair-Kanneganti*\\
  Harvard University \\
  \texttt{aparnank@mit.edu} \\
\And
  Trevor J. Chan*\\
  University of Pennsylvania \\
  \texttt{tjchan@seas.upenn.edu} \\
\And
  Shir Goldfinger\\
  University of Pennsylvania \\
\And
  Emily Mackay\\
  University of Pennsylvania \\
\And
  Brian Anthony\\
  Massachusetts Institute of Technology \\
\And
  Alison Pouch \\
  University of Pennsylvania \\
}

\begin{document}
\maketitle
\let\thefootnote\relax\footnotetext{* equal contribution}

\begin{abstract}
Despite huge advances, LLMs still lack convenient and reliable methods to quantify the uncertainty in their responses, making them difficult to trust in high-stakes applications. One of the simplest approaches to eliciting more accurate answers is to select the mode of many responses, a technique known as ensembling. In this work, we expand on typical ensembling approaches by looking at ensembles with a variable voting threshold. We introduce a theoretical framework for question answering and show that, by permitting ensembles to "abstain" from providing an answer when the dominant response falls short of the threshold, it is possible to dramatically increase the trustworthiness of the remaining answers. From this framework, we derive theoretical results as well as report experimental results on two problem domains: arithmetic problem solving and clinical-note question-answering. In both domains, we observe that large gains in answer trustworthiness can be achieved using highly restrictive voting ensembles, while incurring relatively modest reductions in response yield and accuracy. Due to this quality, voting ensembles may be particularly useful in applications - such as healthcare and data annotation - that require a high degree of certainty but which may not require that every question receive an automated answer.
\end{abstract}


\section{Introduction}
The last few years have seen an explosion in the capabilities and deployment of AI language models across varied intellectual domains. However, the phenomenon of hallucination - where a model confidently produces erroneous information - represents a persistent obstacle to the continued integration of language models \cite{shen2023chatgpt} in high-stakes environments such as healthcare, where an unfounded response can have life-or-death consequences \cite{ahmad2023creating, pal2023med, mackay2025automated}.

Model hallucination has proven tricky to address. For one, hallucinations stem from diverse sources, including epistemic uncertainty, biased training data, and the probabilistic nature of enterprise generative model inference \cite{zhang2023siren, xiao2021hallucination, he2025nondeterminism}. For another, without explicit tuning, models are generally poor estimators of their own ability \cite{kadavath2022language, lin2022teaching, lin2023generating}. Recent studies suggest that hallucination is an innate feature of well-calibrated models and will continue to present challenges even as models improve \cite{xu2024hallucination, rawte2023troubling}.

In this work, we argue for the use of voting ensembles of large language models (LLMs) to mitigate hallucinations. We provide a formal framework to characterize the LLM question space and derive under this framework the theoretical voting ensemble behavior and corresponding response properties. We validated these results by conducting experiments with simulated ensembles of agents, LLMs answering synthetic arithmetic questions, and LLMs answering questions on clinical data in context. Our findings show that ensemble voting is a viable strategy to improve the accuracy and trustworthiness of current state-of-the-art LLMs on real-world problems. 


\section{Related work}

In general, the robustness of LLMs may be improved by training-time and inference-time modifications. The former class includes methods such as increasing the model scale \cite{kaplan2020scaling, hoffmann2022training}, increasing the quality of training data \cite{chang2022data, marion2023less, chen2023alpagasus, tirumala2023d4, li2023quantity}, model fine-tuning and alignment \cite{zhou2024lima, bai2022training, ouyang2022training, wang2022self}, and mixture-of-experts architecture \cite{shazeer2017outrageously, dai2024deepseekmoe}. While undoubtedly effective, these methods incur high costs in data, compute power, and labor.

For these reasons, a large family of methods focuses on modifying the behavior of trained models at inference time in order to elicit better performance without requiring additional training or data. Early papers demonstrate the use of search-based methods to maximize the likelihood of text generated by a language model \cite{wu2016google, yang2018breaking}. Updated implementations of this approach have been proposed for state-of-the-art models as well \cite{meister2020if, xiao2021hallucination, xie2024self}. 

To address the oft-cited weaknesses of poor factual recall and limited specific domain knowledge, LLM agents may incorporate retrieval-augmented-generation \cite{lewis2020retrieval, shi2023replug} or similar techniques which rely upon access to auxiliary resources such as calculators, coding interfaces, and web search \cite{peng2023check, gao2022rarr}. That said, some of the simplest methods are also the most effective. For example, chain of thought prompting, in which a model is directed to generate intermediate reasoning steps on the path towards an answer \cite{wei2022chain, wang2022self, suzgun2022challenging}, is essential to frontier models' ability to solve problems and perform logical inference \cite{liu2024deepseek}. A conceptually similar approach asks a model to predict its own confidence in its generated responses, enabling a degree of uncertainty estimation \cite{kadavath2022language, lin2022teaching, lin2023generating, manakul2023selfcheckgpt}. 

Another common technique makes use of response aggregation in order to improve predictions. Options for aggregation include simply averaging independent responses \cite{hou2023decomposing, wang2022self, greenblatt2024getting}, repeatedly sampling and then verifying responses \cite{brown2024large}, and estimating response uncertainty on the basis of semantic variance across a set of responses \cite{kuhn2023semantic, farquhar2024detecting}. While these methods do not fundamentally change the inference behavior of models, they all serve to estimate model uncertainty, making them particularly useful for improving the reliability of LLMs. 


\section{Preliminaries}

\textbf{Definitions:} We define a question $q \in Q$ as a singular query with a discrete set of potential responses and a single correct response. In practice, the number of potential responses can be finite and small (such as true/false questions) or countably infinite (as with arithmetic problems). This definition excludes completely open-ended queries ("Write an original poem."). While not a focus of this work, previous studies have classified open-ended responses based on semantic entropy \cite{kuhn2023semantic} and semantic clustering \cite{farquhar2024detecting}. We treat each LLM agent as a map from $q$, belonging to problem domain $D \subset Q$, to response $a$ in range $A$. (An example of $D$ from arithmetic is the domain of questions querying whether a product of two integers has 6 digits; here each $q \in Q$ maps to $A := \{0, 1\}$.) An ensemble is defined here as a homogeneous collection of independent agents with a fixed voting criterion. Ensembles can collectively map a question to a response or to a no-consensus state $F: \mathcal{Q} \to \mathcal{A} \cup \{NC\}$. 

\textbf{Notation:} We use capital letters ($C$, $I$, and $NC$) to denote the three possible outcomes of querying a single agent or ensemble. The probabilities of these events are denoted $P(C)$ or equivalently $P_C$. For cases involving multiple questions or problem domains, we use $\bar{P}(C)$, $\bar{P}(I)$, and $\bar{P}(NC)$ to represent the arithmetic mean of probability over all questions.


\section{Proposed framework}
We often think of the average response to a given question of many agents as more accurate than the response of any single agent, and indeed in many real-world scenarios, we do observe benefits to response averaging. However, this is not always the case. An illustrative example is as follows: we do not expect that a group of individuals can predict the outcome of a fair coin flip any better than a single individual, no matter how large the group. To understand the circumstances under which ensembling provides a benefit to accuracy, we must understand what makes a question difficult.

Suppose for simplicity's sake that responses to a question fall into three possible categories. For a given question $q$,
\begin{itemize}[leftmargin=12pt]
   \item[] Let \( A \) be the set of all possible answers.
   \item[] Let \( C \in A \) be the correct answer.
   \item[] Let \( I \in A \) be the dominant specious answer (an answer that appears correct to the agent but is actually incorrect).
   \item[] Let \( R = A \setminus \{C, I\} \) be the set of all other incorrect answers. Some of these may also be specious, but they are selected less frequently.
\end{itemize}

\begin{wrapfigure}[24]{r}{0.3\linewidth}
    \centering
    \includegraphics[width=1\linewidth]{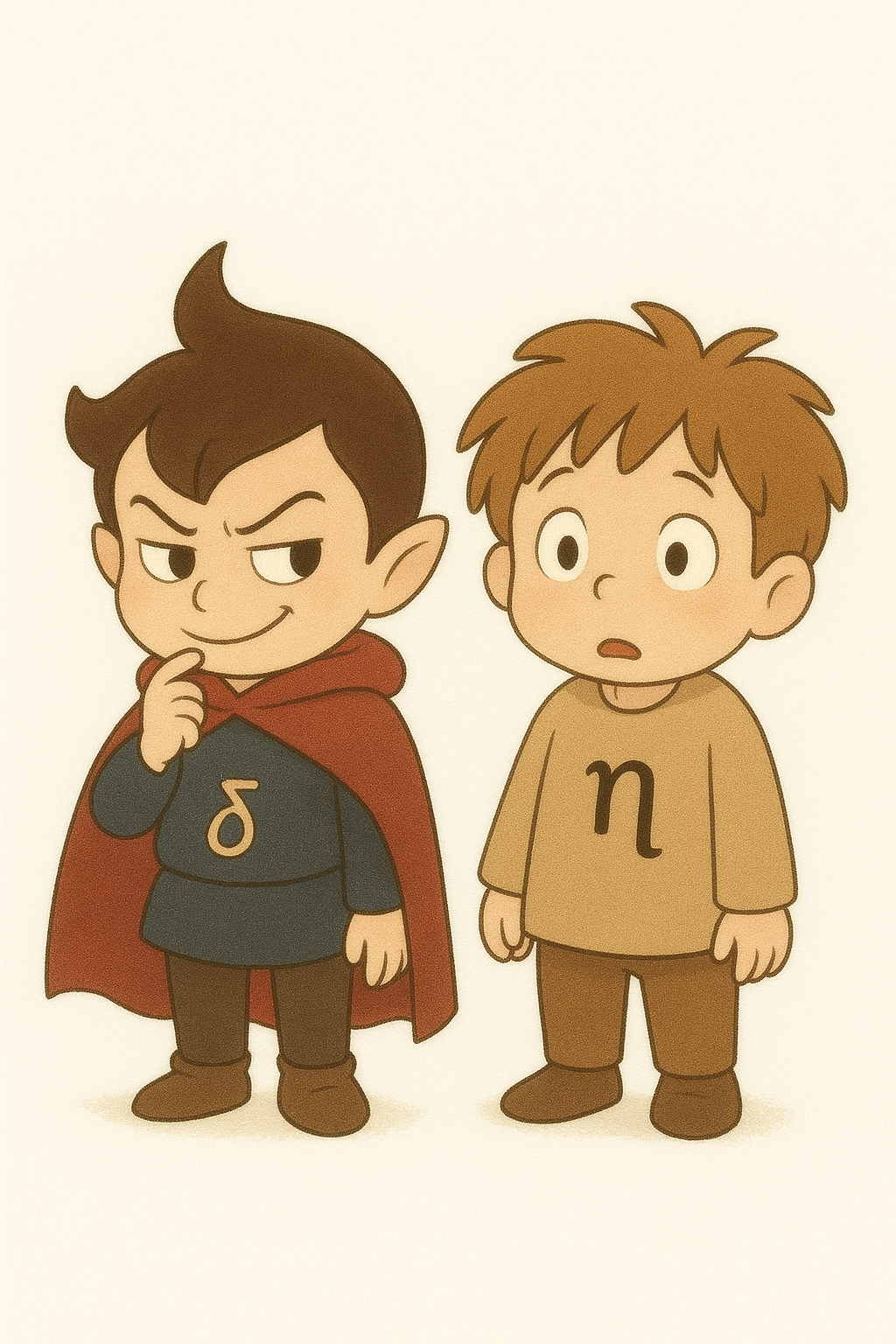}
    \caption{Question difficulty is a function of deceptiveness ($\delta$), a question's tendency to elicit a single specious answer, and bewilderment ($\eta$), the degree to which it encourages random guessing. (Illustration generated by ChatGPT.)}
    \label{fig:ensemble_accuracy}
\end{wrapfigure}

We characterize the response behavior to a question $q$ along the orthogonal axes of deceptiveness ($\delta$) and bewilderment ($\eta$)\footnote{Defined here with respect to a specific question-agent pair. The same question posed to a different agent will be characterized by a different $\delta$ and $\eta$.}:
\begin{itemize}[leftmargin=12pt]
    \item[] Deceptiveness (\( \delta \in [0,1]\)) is the probability that the agent chooses the specious answer \( I \), over the correct answer \( C \) or other answers from \( R\). For \(\delta=0\), the agent always selects the correct answer \( C \). For \( \delta = 1 \), the agent always selects the specious answer \( I \).
    \item[] Bewilderment (\( \eta \in [0,1]\)) is the probability that the agent selects an incorrect answer from \( R \), provided that the agent has not already been deceived. For \( \eta = 0 \), the agent never responds from \(R\). For \( \eta = 1 \), the agent always provides an incorrect answer \textit{that is not $I$}.
\end{itemize}

Intuitively, deceptiveness describes the tendency of a question to mislead an agent by presenting incorrect answers as plausible. Examples of highly deceptive questions are those designed to encourage fallacious logic, such as the Linda problem. These questions trip up both humans and LLMs, which are trained to emulate human language and reasoning. At the same time, some questions which deceive LLMs are trivial to humans (“How many ‘r’s are in the word strawberry?”). Bewilderment quantifies the tendency of a question to force an agent to guess randomly at the answer. Questions with high bewilderment may have a large number of plausible answers (“How many jelly beans are in this Mason jar?”). Both agents and humans will typically answer in an incorrect fashion, though responses may still cluster around a collapsed set of incorrect modes (in this case, "five hundred" may appear more frequently than "1 billion"). We define the difficulty $d$ of a problem as \(1-P(C)\). This can be expressed purely in terms of $\delta$ and $\eta$: \(d = 1 - (1-\delta)(1-\eta) \).

Under this framework, the probability that a single agent yields the correct response becomes \(P(C) = 1 - d = (1 - \eta) (1 - \delta)\). The probability that the agent is deceived and responds speciously is \(P(I) = (1 - \eta) \delta\). Finally, the probability that the agent is bewildered and provides a different incorrect response is \(\sum_i P(r_i \in R) = \eta\).

Now suppose \( n \) agents are independently queried. Let \( k \in [1..n]\) indicate the voting threshold, or the minimum number of agents that must agree on the same response for a consensus to be reached. By definition, responses from $R$ accumulate fewer votes than $I$ and cannot contribute to a consensus, so if neither \( C \) nor \( I \) receives at least \( k \) votes, the ensemble response is \(NC\), or "No consensus." Let \( X_C \) be the number of agents who respond with \( C \) and \( X_I \) be the number of agents who respond with \( I \). Then the conditions for each outcome are as follows:
\begin{itemize}[leftmargin=12pt]
\item[] A consensus around \( C \) occurs if \( X_C \geq k \) and \( X_C > X_I \).
\item[] A consensus around \( I \) occurs if \( X_I \geq k \) and \( X_I > X_C \).
\item[] No consensus occurs if neither \( X_C \geq k \) nor \( X_I \geq k \), or if \( X_I = X_C \).
\end{itemize}
In the case of identical agents, the probability of each outcome reduces to a sum of multinomial PMFs parameterized primarily by the voting threshold, and secondarily by the ensemble size:

\begin{align*}
P(C) = P_C(k, n) &= \sum_{c=k}^{n}\sum_{i=0}^{\min(c-1,n-c)} \binom{n}{c, i, n-c-i}\left((1-\eta)(1-\delta)\right)^c\left((1-\eta)\delta\right)^i\eta^{n-c-i}\\
P(I) = P_I(k, n) &= \sum_{i=k}^{n}\sum_{c=0}^{\min(i-1,n-i)} \binom{n}{c, i, n-c-i} \left((1-\eta)(1-\delta)\right)^c\left((1-\eta)\delta\right)^i\eta^{n-c-i}\\
P(NC) = P_{NC}(k, n) &= \sum_{c=0}^{k-1}\sum_{i=0}^{k-1} \binom{n}{c, i, n-c-i} \left((1-\eta)(1-\delta)\right)^c\left((1-\eta)\delta\right)^i\eta^{n-c-i}\\
      &+ \sum_{\beta=k}^{fl(n/2)}\binom{n}{\beta, \beta, n-2\beta} \left((1-\eta)(1-\delta)\right)^\beta\left((1-\eta)\delta\right)^\beta\eta^{n-2\beta}\\
      &= 1 - P(C) - P(I)
\end{align*}

For an outcome $O$, we will use the notations $P(O), \ P_O(k)$, and $P_O(k, n)$ interchangeably. We compute $P(C)$ as the joint probability that the number of correct responses is \textit{greater than} the threshold $k$, and is also the predominant response; likewise for $P(I)$. In practice, we calculate $P(NC)$ as the complement of $P(C) \cup P(I)$, thereby excluding ties and outcomes in which no response reaches the voting threshold. \footnote{Ties occur when both correct and incorrect surpass the minimum threshold but neither is dominant. It is possible but not always useful to break a tie. We discuss tiebreakers in appendix section \ref{sec:tiebreakers}.}

Voting strategies lie on a spectrum from restrictive to permissive. Strategies with a threshold-to-size ratio (\(\frac{k}{n}\)) close to 1 require high agreement between agents in the ensemble to reach a consensus and are comparatively restrictive, while strategies with a \(\frac{k}{n}\) close to zero are comparatively permissive. In practice, strategies range from plurality voting (\(k=1\)), also referred to as first-past-the-post voting, as the most permissive, and unanimous voting (\(k=n\)), requiring total agreement, as the most restrictive. We will use the words "restrictive" and "permissive" to refer to both voting strategies themselves and the resulting ensemble behavior.

In this work, we define the following performance criteria: 
\begin{itemize}[leftmargin=12pt]
\item[] The probability of an ensemble obtaining a correct response, denoted \textit{accuracy}.
\item[] The probability that any consensus answer is correct, denoted \textit{trust}, $T(k)$.
\item[] The probability that a question receives a consensus answer as opposed to a verdict of “no-consensus", denoted \textit{yield}, $Y(k)$.
\end{itemize}

\setlength{\tabcolsep}{18pt} 
\renewcommand{\arraystretch}{1.5} 
\begin{table}[h]
    \centering
    \caption{Three criteria of performance from a voting ensemble}
    \begin{tabular}{ccc}
        Accuracy & Trust & Yield \\ \hline
        $P(C)$ &  $P(C|\sim NC)$ & $P(\sim NC)$ \\
         &  $\frac{P(C)}{P(C)+P(I)}$ & $P(C)+P(I)$
    \end{tabular}
    \label{tab:performance}
\end{table}


\section{Theoretical results}
We can derive a number of theoretical results from this ensemble voting framework. Specifically, we examine the effect of ensemble composition, which comprises size and voting threshold, on the expected accuracy, trust, and yield for various abstract problem domains.

\begin{figure}[t]
    \centering
    \includegraphics[width=1\linewidth]{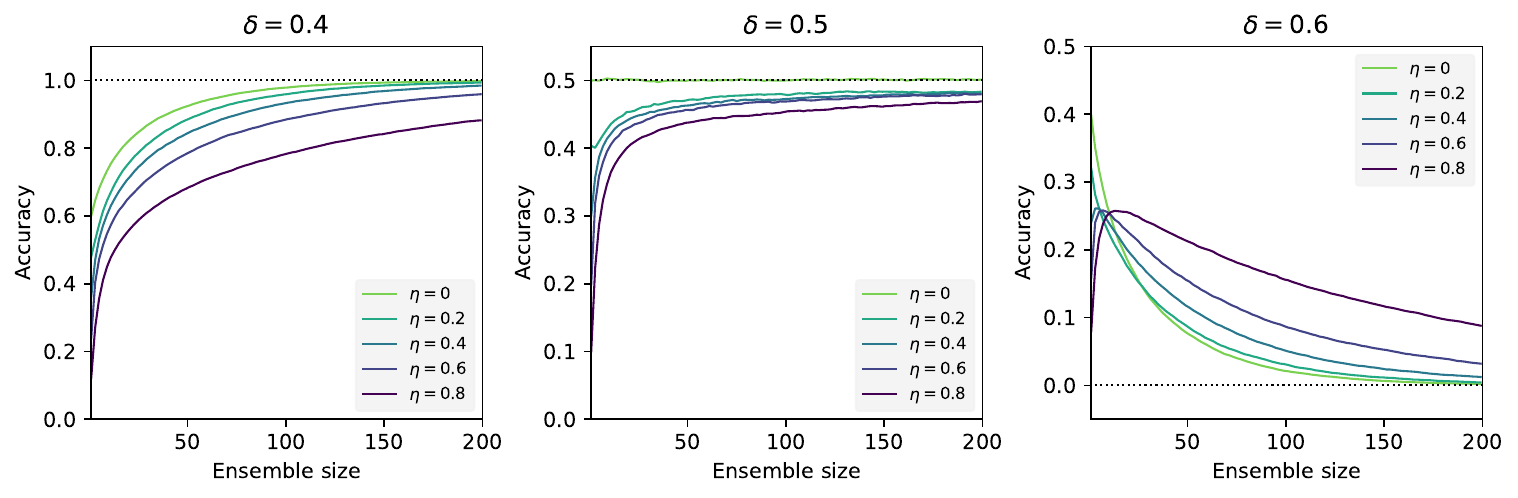}
    \caption{Simulated voting scenarios show  show that for a single question, deceptiveness $\delta$ alone dictates the ensemble accuracy as ensemble size approaches infinity. However, the rate of convergence is governed by the bewilderment.}
    \label{fig:ensemble_accuracy}
\end{figure}


Firstly, we note that restrictive ensembles differ from permissive ensembles only in the proportion of questions for which they reach a consensus, rather than at the level of single-agent responses. From this, we arrive at the following: 

\begin{restatable}[Accuracy maximization under permissive voting]{thm}{accmax}
\label{thm:accmax}
Accuracy is monotonic decreasing with respect to $k: P_C(k+1) \leq P_C(k) \ \forall \ k\in[1,...,n-1]$.
\end{restatable}
Increasing the threshold $k$ of a voting ensemble can only maintain or decrease the proportion of questions that are answered correctly. Therefore, the optimal ensemble for maximizing accuracy is always the most permissive, all else being equal. Using almost identical logic, we show that ensemble yield, which counts both correct and incorrect answers, follows a similar pattern.

\begin{restatable}[Yield maximization under permissive voting]{thm}{ymax}
\label{thm:ymax}
Yield is monotonic decreasing with respect to $k$: $Y(k) \leq Y(k+1) \ \forall \ k\in[1,..,n-1]$.
\end{restatable}

Proofs are provided in Appendix section \ref{restrict}. 


Next, we show that for very large ensembles, the maximum achievable accuracy for any given question is fixed by the question deceptiveness $\delta$.
\begin{restatable}[Maximal accuracy in the large-ensemble limit]{thm}{acclim}
\label{thm:acclim}
\begin{equation*}
    \lim_{n \to \infty} P_c(k_{opt}, n) = 
    \begin{cases}
        1, &\delta < 0.5 \\
        0.5, &\delta = 0.5 \\
        0, &\delta > 0.5
    \end{cases}
\end{equation*}
where $k_{opt}$ is the choice of threshold $k$ that maximizes accuracy. 
\end{restatable}
Our first result establishes the size-agnostic, optimal voting threshold $k=1$. We can show that the probability of reaching a no-consensus outcome goes to zero for very large ensembles. The first way to obtain the no-consensus outcome is if every answer falls short of the voting threshold $k$. For $k=1$, the probability of this phenomenon is of course zero. The second way is if a tie occurs between the correct and specious responses. We show that this probability likewise goes to zero, by establishing an upper bound on the multinomial summand and applying the squeeze theorem.

We turn our attention now to the behavior of $P_C$ for large ensembles, and we construct three cases dependent on $\delta$. When$\delta=0.5$, the probabilities $P_C$ and $P_I$ are formulated identically, so, given the impossibility of a no-consensus outcome, both $P(C) = P(I) =0.5$. Furthermore, we know that for very large $n$, the PMFof the multinomial distribution of answers exhibits a single sharp mode at the intrinsic probability vector. For values of $\delta < 0.5$, the state vector represented by this mode is incompatible with achieving consensus on an incorrect answer, so $P_I$ goes to 0 and $P_C$ goes to 1. By symmetry, the opposite is true when $\delta > 0.5$. Proofs are provided in appendix section \ref{proofs}. Consequently, for a very large and maximally permissive ensemble, accuracy is perfect for those questions where $\delta < 0.5$. 


\section{Experimental results}


\subsection{Arithmetic}
As an initial experiment to assess ensembles of language model, we constructed synthetic datasets of arithmetic problems across a range of difficulties. Arithmetic is a common LLM test case\cite{zhang2024careful, imani2023mathprompter, qiu2024dissecting, yang2023gpt}, since the complexity of arithmetic problems is well-defined and can be roughly estimated. (In general, problems with more digits and/or operations are harder than those with fewer digits and/or operations). We constructed ensembles of independent instances of the publicly available model Llama3-70B-instruct, from Meta \cite{dubey2024llama}. Generation parameters were kept at the default values.

Two categories of arithmetic problems were used. The first is multi-digit multiplication problems, in which an expression is given in the form \textcolor{purple}{$a * b$}. The second is order-of-operations problems, in which operators sampled from $\left\{+,-,*\right\}$ and parentheses are concatenated. (An example expression would be \textcolor{purple}{$((8 + 4) * 3) + (4 * 6) - 9 - 6$}.) In both cases, the models were directed to follow the prompt: \texttt{“Evaluate the following arithmetic expression and output the final answer as a single number. Do not show your work. Do not include any extra text. \{\textcolor{purple}{expression}\} = ?”}. Only model responses containing a single integer number were included, to ensure that chain-of-thought generation did not affect ensemble results.

\begin{figure}[t]
    \centering
    \begin{subfigure}[h]{1\textwidth}
        \centering
        \includegraphics[width=1\linewidth]{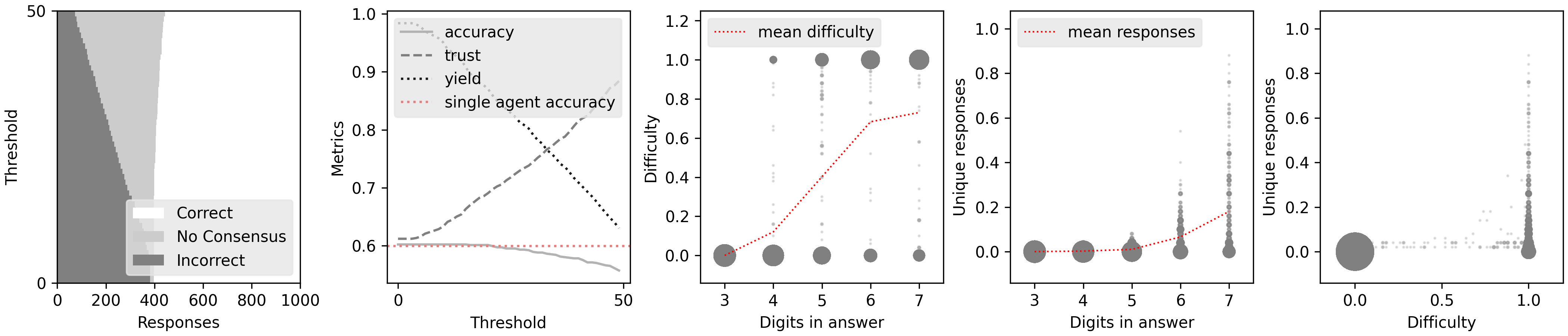}
        \caption{}
    \end{subfigure}
    \hfill
    \begin{subfigure}[h]{1\textwidth}
        \centering
        \includegraphics[width=1\linewidth]{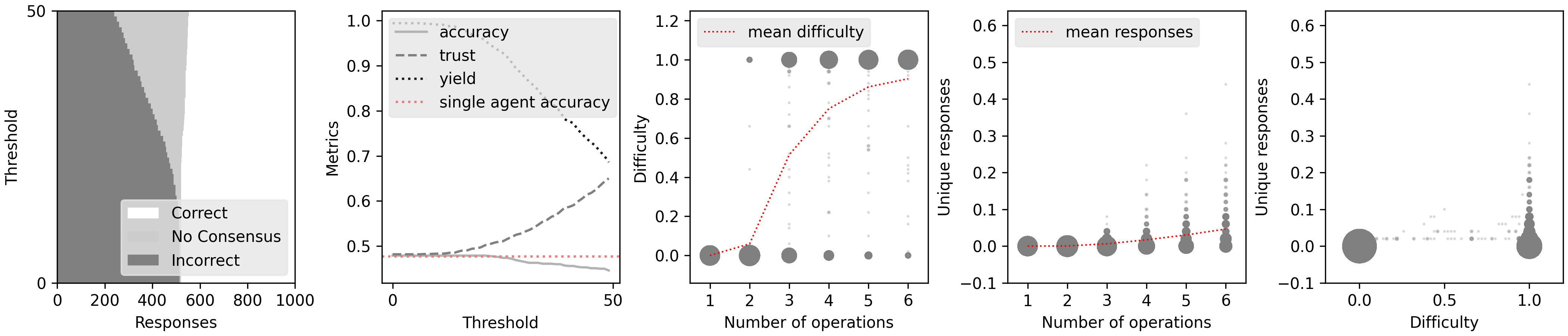}
        \caption{}
    \end{subfigure}
    \caption{(a) An ensemble of 50 models answers multiplication questions of varying difficulty. Notably, trust improves at high voting thresholds. (b) Similar behavior is observed when the ensemble evaluates arithmetic expressions with multiple operations.}
    \label{fig:arithmetic}
\end{figure}

\begin{figure}[h]
    \centering
    \begin{subfigure}[h]{1\textwidth}
        \centering
        \includegraphics[width=1\linewidth]{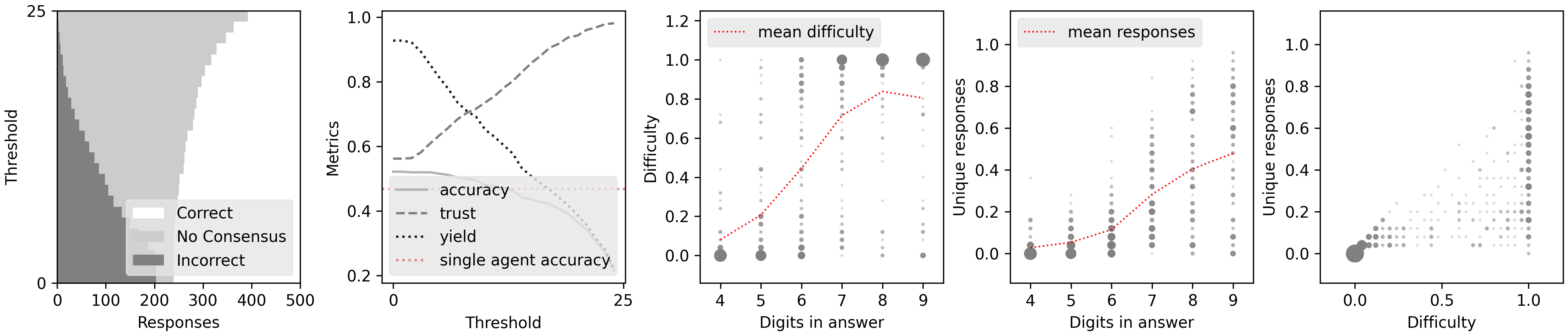}
        \caption{}
    \end{subfigure}
    \hfill
    \begin{subfigure}[h]{1\textwidth}
        \centering
        \includegraphics[width=1\linewidth]{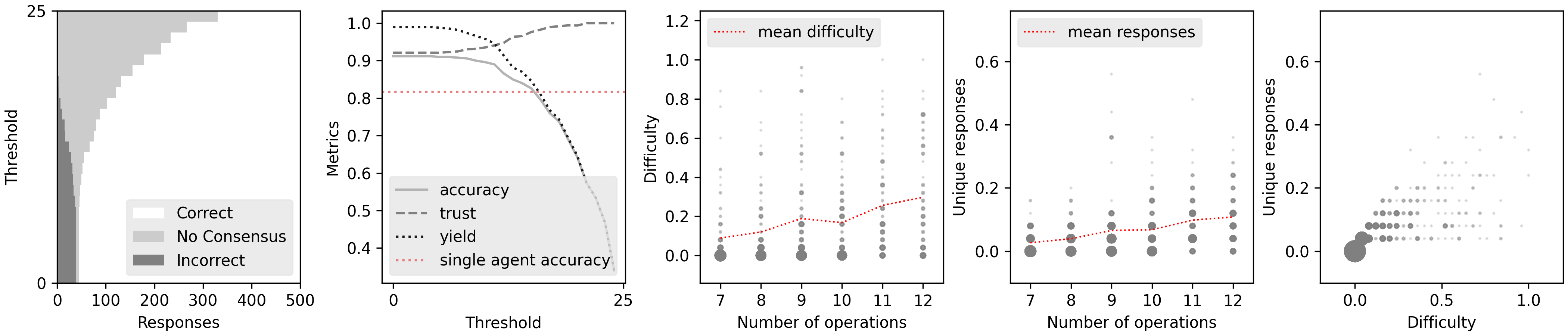}
        \caption{}
    \end{subfigure}
    \caption{An ensemble answers arithmetic questions using chain-of-thought prompting. Performance and question distributions are shown for questions involving (a) multiplication and (b) expressions with multiple operations.}
    \hfill
    \label{fig:cot}
\end{figure}

The results of both experiments show that ensembling is a valid strategy for increasing trust, but that it provides no benefit to accuracy over using a single model (Figure \ref{fig:arithmetic}). For each of the questions posed, we empirically estimated the difficulty $d$ as the fraction of incorrect answers. As expected, multiplication questions involving more digits and order-of-operations questions comprising more operations proved more difficult for the models to answer, and elicited a greater diversity of unique responses. Trust significantly increases with voting restrictiveness (from 0.61 to 0.88 for multiplication questions and from 0.48 to 0.65 for order-of-operations questions).

Next, we sought to understand whether these benefits hold when using a chain-of-thought (cot) prompting strategy, which is now understood to be a practical necessity for LLMs solving difficult problems. We performed two additional experiments on arithmetic questions to examine how cot affects the distribution of ensembled responses (figure \ref{fig:cot}). Predictably, cot does allow the model (Llama3-70B-instruct) to answer much harder questions. In previous experiments the model is explicitly prevented from showing work and achieves <50\% accuracy answering multiplication questions involving more than 4 digits or order-of-operations questions with more than 2 operations. When instructed to show work, the model achieves over 50\% accuracy on multiplication questions up to 6 digits and order of operations questions exceeding 12 operations. Apart from uniformly increasing model performance, cot seems to not significantly affect the general behavior of the ensemble. In particular, the benefits of an ensembling approach are still present: accuracy is increased through voting compared to a single model, and trust can be further improved through restrictive voting schemes.

\subsection{Clinical Reports}
\label{sec:clinical_reports}
To understand whether this behavior carries over to problems of practical interest, we tested our ensembles' ability to extract detailed clinical information from dense perioperative patient notes. Data consisted of 500 text reports collected by cardiac surgeons and cardiac anesthesiologists between 2017 and 2022 at both the Hospital of the University of Pennsylvania and the Pennsylvania Presbyterian Hospital. (An example excerpt is provided in the supplement section \ref{sec:example_report}.) For each report, we asked the Llama3-8B-instruct models to extract three notable preoperative patient characteristics: left ventricular ejection fraction, severity of mitral valve stenosis, and severity of mitral valve regurgitation. Model temperature was raised from a default value of 0.6 to 1, and all other generation parameters were kept at the default values. (A discussion of model temperature is given in Appendix section \ref{sec:temperature}.) Echocardiogram reports were manually parsed to obtain ground-truth labels.

\begin{figure}[h]
    \centering
    \begin{subfigure}[h]{1\textwidth}
        \centering
        \includegraphics[width=1\linewidth]{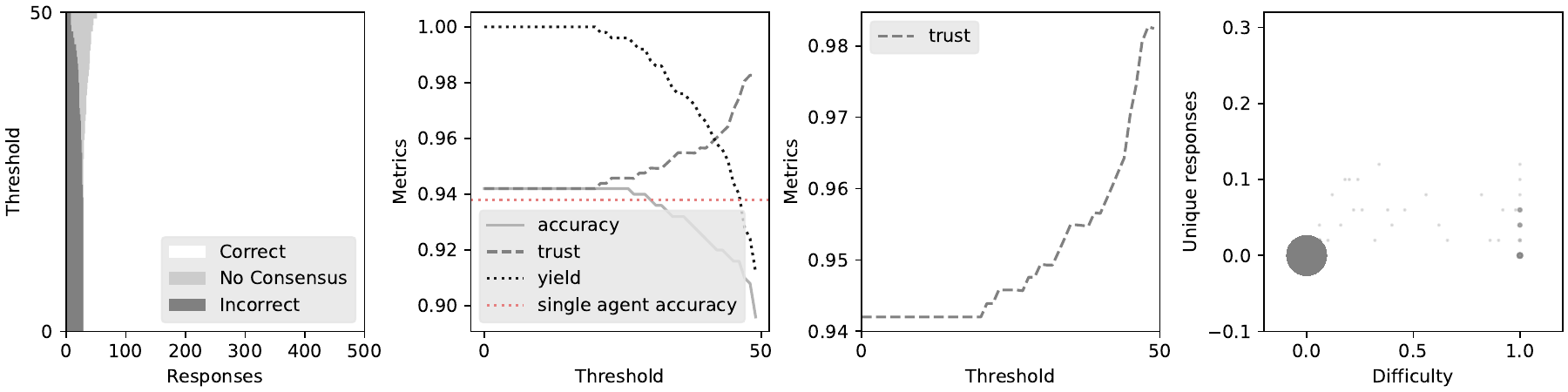}
        \caption{}
    \end{subfigure}
    \hfill
    \begin{subfigure}[h]{1\textwidth}
        \centering
        \includegraphics[width=1\linewidth]{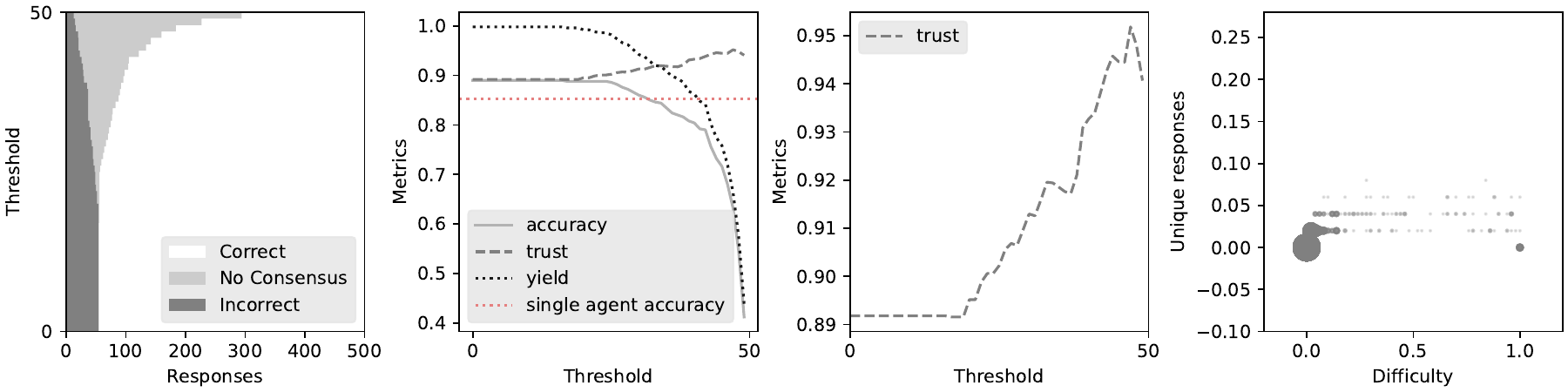}
        \caption{}
    \end{subfigure}
    \hfill
    \begin{subfigure}[h]{1\textwidth}
        \centering
        \includegraphics[width=1\linewidth]{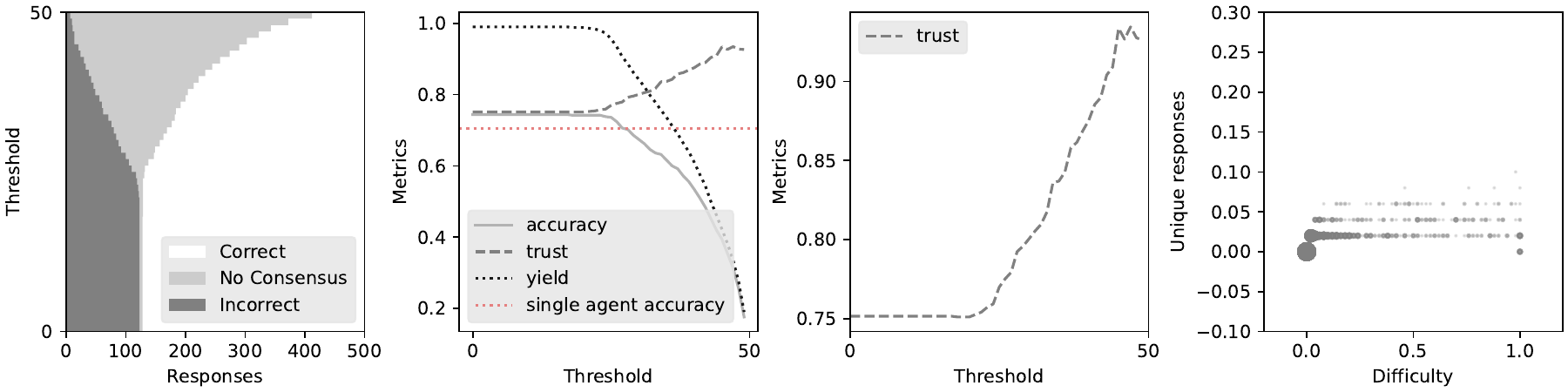}
        \caption{}
    \end{subfigure}
    \caption{A single model and ensemble  extract salient features from the text of echocardiogram reports. Accuracy, yield, and trust are shown as a function of voting threshold and question distributions for model-extracted (a) Left ventricular ejection fraction (LVEF), (b) mitral stenosis (MS), and (c) mitral regurgitation (MR).}
    \label{fig:clinical_results}
\end{figure}

As with the arithmetic  domain, ensembling produces multiple benefits. We see stable-to-marginally improved answer accuracy compared to that of a single model (LVEF: 0.94 to 0.94, MS: 0.85 to 0.89, MR: 0.70 to 0.74). Most notably, trust increases markedly across all three features as ensembles become more restrictive (LVEF: 0.94 to 0.98, MS: 0.85 to 0.94, MR: 0.70 to 0.93).


\section{Discussion}

The prospect of knowing when a model is or is not confident in its answer has inspired a broad array of work in uncertainty estimation for LLMs. While many approaches take a granular look within the network \cite{lindsey2025biology, ferrando2024know} or seek to break down the process of answering a single question into many more transparent steps \cite{radhakrishnan2023question}, our approach aligns generally with methods for estimating uncertainty based only on the output of the model forward pass. These include works aimed at computing informational entropy over a range of responses \cite{farquhar2024detecting, kuhn2023semantic} and for decomposing aleatoric and epistemic sources of ambiguity through ensembling \cite{hou2023decomposing}. Compared to these approaches, the much simpler ensembling approach presented here is less concerned with understanding the source of uncertainty and more concerned with providing tools to measure uncertainty and flexibly mitigate it.

Under our framework, simply counting the fraction of responses that are not either the correct response or the dominant (highest frequency) incorrect response is a straightforward and effective way to estimate the bewilderment $\eta$. Deceptiveness, $\delta$, is likewise easy to estimate as the fraction of dominant incorrect responses over the sum of both correct and dominant incorrect responses. Combined, these are sufficient to model the expected behavior of various perturbations to the ensemble, such as varying threshold or size, and to optimize for competing demands of reducing inference compute, maximizing accuracy, and maximizing trust in any individual answer.

\section{Conclusions, limitations, and future work}
In this work we introduce a theoretical framework for question answering and explore its implications for the behavior of multi-agent ensembles. Our analysis elucidates theoretical bounds on the utility benefits of ensemble voting. Specifically, we outline types of questions and question domains in which improvements in accuracy can be achieved, chiefly at the cost of compute, and in which improvements to trust can be achieved at the cost of response yield. These results are further demonstrated in the case of solving arithmetic problems and answering clinical questions using LLMs, where we show that implementing highly restrictive voting schemes can vastly increase trust and reduce hallucinations in the ensemble's predictions. While detecting hallucinations is one obvious practical application of this work, higher thresholds necessarily decrease yield, so a determination must be made by the user as to what an acceptable hallucination rate is given a specific task. Additionally, as a simplistic yet rigorous method for trading test time compute for additional model performance, ensemble voting could be a useful methodological baseline for similar methods aimed at improving the performance of pretrained models.

\section*{Acknowledgments}
This work was supported by the NIH National Heart, Lung, and Blood Institute under award number R01-HL163202 and the National Science Foundation under award number DGE-2236662.


\bibliographystyle{unsrt}  
\bibliography{references}  


\appendix

\section{Proofs}\label{proofs}

\subsection{Restrictive voting reduces ensemble accuracy and yield}\label{restrict}

\accmax*

\begin{proof}
It is clear that for $k\in[1..n-1]$, 

\begin{align*}
    P_C(k) &= \sum_{c=k}^{n}\sum_{i=0}^{\min(c-1,n-c)} \frac{n!}{c!n!(n-c-i)!} \left((1-\eta)(1-\delta)\right)^c\left((1-\eta)\delta\right)^i\eta^{n-c-i} \\
    &= \sum_{c=k+1}^{n}\sum_{i=0}^{\min(c-1,n-c)} \frac{n!}{c!n!(n-c-i)!} \left((1-\eta)(1-\delta)\right)^c\left((1-\eta)\delta\right)^i\eta^{n-c-i} \\ 
    &+ \sum_{i=0}^{\min(k-1,n-k)} \frac{n!}{k!n!(n-k-i)!} \left((1-\eta)(1-\delta)\right)^k\left((1-\eta)\delta\right)^i\eta^{n-k-i} \\
    &= P_C(k+1) + \sum_{i=0}^{\min(k-1,n-k)} \frac{n!}{k!n!(n-k-i)!} \left((1-\eta)(1-\delta)\right)^k\left((1-\eta)\delta\right)^i\eta^{n-k-i}
\end{align*} 

Let the remaining series on the RHS be $s$; clearly, the summand is strictly non-negative, so $s \geq 0$. 
Then $P_C(k) \geq P_C(k+1) + (0) = P_C(k+1)$. 
\end{proof}

\ymax*
\begin{proof}
$Y(k) = P_C(k) + P_I(k)$. By \cref{thm:accmax}, $P_C(k)$ is monotonic decreasing. We can use identical logical to show that $P_I(k)$, which takes analogous form to $P_C(k)$, is likewise monotonic decreasing. Therefore $Y(k)$ is monotonic decreasing.
\end{proof}

\subsection{Maximal accuracy of large ensembles is solely governed by question deceptiveness}\label{maximal}

\begin{lemma}[No no-consensus in the large-ensemble limit]
\label{lemma:pnc}
When $k = 1, \ \lim_{n \to \infty} P_{NC}(k, n) = 0 \ \forall \ \delta \in [0,1], \ \eta \in [0,1)$.
\end{lemma}
\begin{proof}
Recall that 
\begin{align*} 
P_{NC}(k, n) &= \sum_{c=0}^{k-1}\sum_{i=0}^{k-1} \frac{n!}{c!i!(n-c-i)!} \left((1-\eta)(1-\delta)\right)^c\left((1-\eta)\delta\right)^i\eta^{n-c-i} \\ 
&+ \sum_{\beta=k}^{fl(n/2)}\frac{n!}{\beta!\beta!(n-2\beta)!} \left((1-\eta)(1-\delta)\right)^\beta\left((1-\eta)\delta\right)^\beta\eta^{n-2\beta} \\ 
P_{NC}(1, n) &= (0) + \sum_{\beta=1}^{fl(n/2)}\frac{n!}{\beta!\beta!(n-2\beta)!} \left((1-\eta)(1-\delta)\right)^\beta\left((1-\eta)\delta\right)^\beta\eta^{n-2\beta}\\
      &= \sum_{\beta=1}^{fl(n/2)}\frac{n!}{\beta!\beta!(n-2\beta)!} \left((1-\eta)^2(1-\delta)\delta\right)^\beta\eta^{n-2\beta}
\end{align*}

Now let $\beta := bn$, where $b = \{b_1, b_2, ..., b_{fl(n/2)}\}, b_i \in (0, \frac{1}{2}] \ \forall \ i$, such that

\begin{align*}
P_{NC}(1,n) &= \sum_b \frac{n!}{(bn)!(bn)!(n-2bn)!} \left((1-\eta)^2(1-\delta)\delta\right)^{bn}\eta^{n-2bn}.\\
\end{align*}

We notice that $(1-\delta)(\delta) \leq \frac{1}{4}$. Furthermore, we can show that $(1-\eta)^{2bn}\eta^{n-2bn}$ is maximized when $n = 0, \ \eta = 1, \text{or} \ \eta = 1-2b$. The first 2 maximizing conditions are inherently disallowed, and thus the summand

\begin{align*}
\frac{n!}{(bn)!(bn)!(n-2bn)!} \left((1-\eta)^2(1-\delta)\delta\right)^{bn}\eta^{n-2bn} &\leq 
\frac{n!}{(bn)!(bn)!(n-2bn)!} \left((2b)^2\frac{1}{4}\right)^{bn}(1-2b)^{n-2bn} \\
&= \frac{n!}{(bn)!(bn)!(n-2bn)!} (b^2)^{bn}(1-2b)^{n-2bn} \\
&= \frac{n!}{(bn)!(bn)!(n-2bn)!} \left(b^{2b}(1-2b)^{1-2b}\right)^n \\
\text{Applying the Stirling approximation for factorials:} \\
&\approx \frac{\sqrt{2\pi n}\left(\frac{n}{e}\right)^n}{\left(\sqrt{2\pi bn}\left(\frac{bn}{e}\right)^{bn}\right)^2 \sqrt{2\pi (1-2b)n}\left(\frac{(1-2b)n}{e}\right)^{(1-2b)n} }  \left(b^{2b}(1-2b)^{1-2b}\right)^n \\
&= \frac{\left((1-2b)^{2b-1}b^{-2b}\right)^n}{2 \pi bn (1-2b)^{\frac{1}{2}}} \left(b^{2b}(1-2b)^{1-2b}\right)^n  \\
&= \frac{1}{2\pi b n (1-2b)^\frac{1}{2}} \\
\end{align*}

Then
\begin{align*}
\ P_{NC}(1, n) &\leq \sum_b \frac{1}{2\pi b n (1-2b)^\frac{1}{2}} \\ 
\lim_{n \rightarrow \infty} P_{NC}(1, n) &\leq \lim_{n \rightarrow \infty} \sum_b \frac{1}{2\pi b n (1-2b)^\frac{1}{2}} \\ 
&= \sum_b \lim_{n \rightarrow \infty} \frac{1}{2\pi b n (1-2b)^\frac{1}{2}} \\
&= \sum_b 0 \\
&= 0
\end{align*}

(We can shift the limit inside of the RHS sum because the upper-bounding series is clearly uniformly convergent.) But $P(NC)$ is strictly non-negative, as it represents a probability. So $0 \leq \lim_{n \rightarrow \infty} P(NC) \leq 0$, and thus by the squeeze theorem, $\lim_{n \rightarrow \infty} P(NC) = 0$. 
\end{proof}

\acclim*

\begin{proof}
From \cref{thm:accmax}, accuracy is maximized for $k  = k_{opt} = 1$.

We begin with the special case of $\delta = 0.5$ and show that $\lim_{n \rightarrow \infty} P_{C}(1, n) = 0.5$.  

When $\delta=0.5$,
\begin{align*}
P_C(1, n) &= \sum_{c=k}^{n}\sum_{i=0}^{\min(c-1,n-c)} \frac{n!}{c!i!(n-c-i)!} \left((1-\eta)(1-\frac{1}{2})\right)^{c}\left((1-\eta)\frac{1}{2}\right)^{i}\eta^{n-c-i}\\
     &= \sum_{c=k}^{n}\sum_{i=0}^{\min(c-1,n-c)} \frac{n!}{c!i!(n-c-i)!} \left(\frac{1-\eta}{2}\right)^{c+i}\eta^{n-c-i}\\
P_I(1,n) &= \sum_{i=k}^{n}\sum_{c=0}^{\min(i-1,n-i)} \frac{n!}{c!i!(n-c-i)!} \left((1-\eta)(1-\frac{1}{2})\right)^{c}\left((1-\eta)\frac{1}{2}\right)^{i}\eta^{n-c-i}\\
     &= \sum_{i=k}^{n}\sum_{c=0}^{\min(i-1,n-i)} \frac{n!}{c!i!(n-c-i)!} \left(\frac{1-\eta}{2}\right)^{c+i}\eta^{n-c-i}
\end{align*}
Let $\alpha = c + i$. Then
\begin{align*}
P_C(1, n) &= \sum_{c=k}^{n}\sum_{\alpha=c}^{\min(2c-1,n)} \frac{n!}{c!(\alpha-c)!(n-\alpha)!} \left(\frac{1-\eta}{2}\right)^{\alpha}\eta^{n-\alpha}\\
P_I(1, n) &= \sum_{i=k}^{n}\sum_{\alpha=i}^{\min(2i-1,n)} \frac{n!}{(\alpha-i)!i!(n-\alpha)!} \left(\frac{1-\eta}{2}\right)^{\alpha}\eta^{n-\alpha}\\
\end{align*}

Since the two series are only differentiated by the choice of indexing variable, they must be identical; that is $P_C(1,n) = P_I(1,n)$. 

$P(C) + P(I) + P(NC) = 1$ by definition, and from \cref{lemma:pnc}, $\lim_{n \rightarrow \infty} P_{NC}(1, n) = 0$. So we must have that 

\begin{align*}
    \lim_{n \rightarrow \infty} \left(P_C(1,n) + P_I(1,n) + P_{NC}(1,n) \right) &= \lim_{n \rightarrow \infty} 2P_C(1,n)
    + \lim_{n \rightarrow \infty} P_{NC}(1,n) \\
    &= 2 \left(\lim_{n \rightarrow \infty} P_C(1,n)\right) + (0) \\
    &= 1
\end{align*}

And thus $\lim_{n \rightarrow \infty} P_C(1,n) = 0.5$.

We now move to the more common case, in which $\delta \neq 0.5$. 



Consider only the summands of $P(C), P(I)$, which sample the multinomial distribution:

\begin{equation*}
    f(X | n, p) = \frac{n!}{\alpha!\beta!\gamma!} \left((1-\eta)(1-\delta)\right)^\alpha \left((1-\eta)\delta\right)^\beta \eta^{\gamma}.
\end{equation*}

Here, $X$ is the state vector $[\alpha, \beta, \gamma]$ (where $\gamma = n-\alpha-\beta$), and $p$ is the corresponding probability vector $[(1-\eta)(1-\delta), (1-\eta)\delta, \eta]$. \\

It is well-known that as $n \rightarrow \infty$, the probability mass of the multinomial distribution concentrates in the immediate neighborhood of $p$, such that $X \rightarrow \mathbb{E}[X] = np$. \\

Clearly $\lim_{n\rightarrow \infty} f_C =  
    \lim_{n\rightarrow \infty} f_I$.


When $\delta < 0.5$, we have that $1 - \delta > \delta$. Thus 

\begin{align}
\lim_{n \rightarrow \infty} c &= n(1-\eta)(1-\delta) \nonumber \\
&> n(1-\eta)\delta \nonumber \\
&= \lim_{n\rightarrow \infty} i 
\end{align}
 
Invoking the definitions of $P_I(1, n)$ and $f_I$: 

\begin{align*}
\lim_{n \rightarrow \infty} P_I(1, n) & = \lim_{n \rightarrow \infty} \sum_{i=k}^{n}\sum_{c=0}^{\min(i-1,n-i)} f(X = [c, i, n-c-i] | n, p) \\
&= 
\sum_{i=k}^{n}\sum_{c=0}^{\min(i-1,n-i)} \lim_{n \rightarrow \infty} f(X = [c, i, n-c-i] | n, p) &\\
\end{align*}

assuming uniform convergence of $P_I(1,n)$. \\

By (1), $\lim_{n\rightarrow\infty} c > \lim_{n\rightarrow\infty} i$. However the bounds on $P_I(1,n)$ restrict our sum to those terms where $i > c$. Thus $P_I(1,n)$ vanishes. As $P_C(1,n) + P_I(1,n) + P_{NC}(1,n) = 1$, $\lim_{n \rightarrow \infty} P(C) = 1$ for $ \delta < 0.5$. \\

When $\delta > 0.5$, we have, inversely,  that $n(1-\eta)\delta > n(1-\eta)(1-\delta)$. It follows by a symmetrical argument to the above that $\lim_{n \rightarrow \infty} P_C(1,n) = 0$ and $\lim_{n \rightarrow \infty} P_I(1,n) = 1$.

\end{proof}

\section{Example echocardiography report}
\label{sec:example_report}

We give an example of a typical echocardiography report of the type analyzed in section \ref{sec:clinical_reports}. For these reports, a language model was used to identify and extract in tabular form several features of interest, including mitral regurgitation (red), mitral stenosis (green), and left ventricular ejection fraction (blue). This exercise resembles the common task of parsing large and loosely formatted electronic health record information in order to guide medical interventions or to perform epidemiological analyses.

\texttt{75 yo M. Pre-intervention support: EPI AV: The aortic valve is tricuspid. It is mildly calcified with severe stenosis and mod to severe AR. MV: Valve exhibits excessive mobility, prolapsing P2 segment with \textcolor{red}{mod-sev regurgitation}. Mean and peak velocity of 19 and 7 mmHg \textcolor{green}{mildly stenotic}. LV: The left ventricle has low normal global systolic function with an \textcolor{blue}{ejection fraction in the range of 40-45\%} and normal chamber size. No ventricular hypertrophy, no regional wall motion abnormalities. RV: The right ventricle has decreased systolic function and normal chamber size. TV: the tricuspid valve has trace regurgitation. LA: The left atrium is normal in size and no thrombus noted in atrial appendage. RA: The right atrium is mildly dilated and pulmonary artery catheter visualized in chamber. IAS: No PFO is noted by Color Flow Doppler. Ao:  The aorta is intact and normal in size: there is moderate to severe atherosclerotic disease. PV: The pulmonic valve has trace regurgitation. Pcard: No significant effusion was noted. Pleura: No significant pleural effusion was noted.}

\section{Tie breakers}
\label{sec:tiebreakers}
In the framework above, we largely gloss over the question of ties, which can occur for threshold values $k<\frac{n}{2}$. In theorem \ref{}, we show that the probability of a tie goes to zero as ensemble size increases. Nonetheless, ties do occur at practical ensemble sizes of 10-50 agents and present a potential dilemma. One option, which we outline above, is to shunt questions with a tied number of responses to the no-consensus category. A second valid approach would be to conduct a 'tie breaker' vote, in which additional agents are queried until the tie can be resolved. In practice, it may not be easy to simply run an unbounded number of additional queries to a language model. (An only slightly inferior method for breaking ties is to choose one of the predominant answers at random.) Whether the former strategy or the latter is more suitable depends on the use case. Generally, for problem domains in which questions have only a singular correct answer and the expected accuracy exceeds 50\%, it is advantageous to use the former.

\section{Temperature}
\label{sec:temperature}
The temperature parameter in LLM inference modifies the token-wise probabilities predicted by a model. A low temperature biases outputs to only the highest probability tokens, while a high temperature dilutes the sampling probability over a larger range of values. Ensembling approaches clearly rely on a temperature > 0 to be effective, but it is not clear what temperature is optimal, or how sensitive these approaches are to temperature. We find that, while increasing or decreasing model temperature affects the accuracy and variability of responses within an ensemble, it rarely has an effect on the final answer of the ensemble. Intuitively, we can understand temperature as modulating question $\eta$ while leaving question $\delta$ the same. This means that, in certain problem domains where the $\eta$ of questions is already very low, increasing model temperature is an effective way to reduce the size of the ensemble necessary to disambiguate between questions with a near-zero $\eta$ and questions with a merely low $\eta$.

\begin{figure}[b]
    \centering
    \begin{subfigure}[h]{1\textwidth}
        \centering
        \includegraphics[width=1\linewidth]{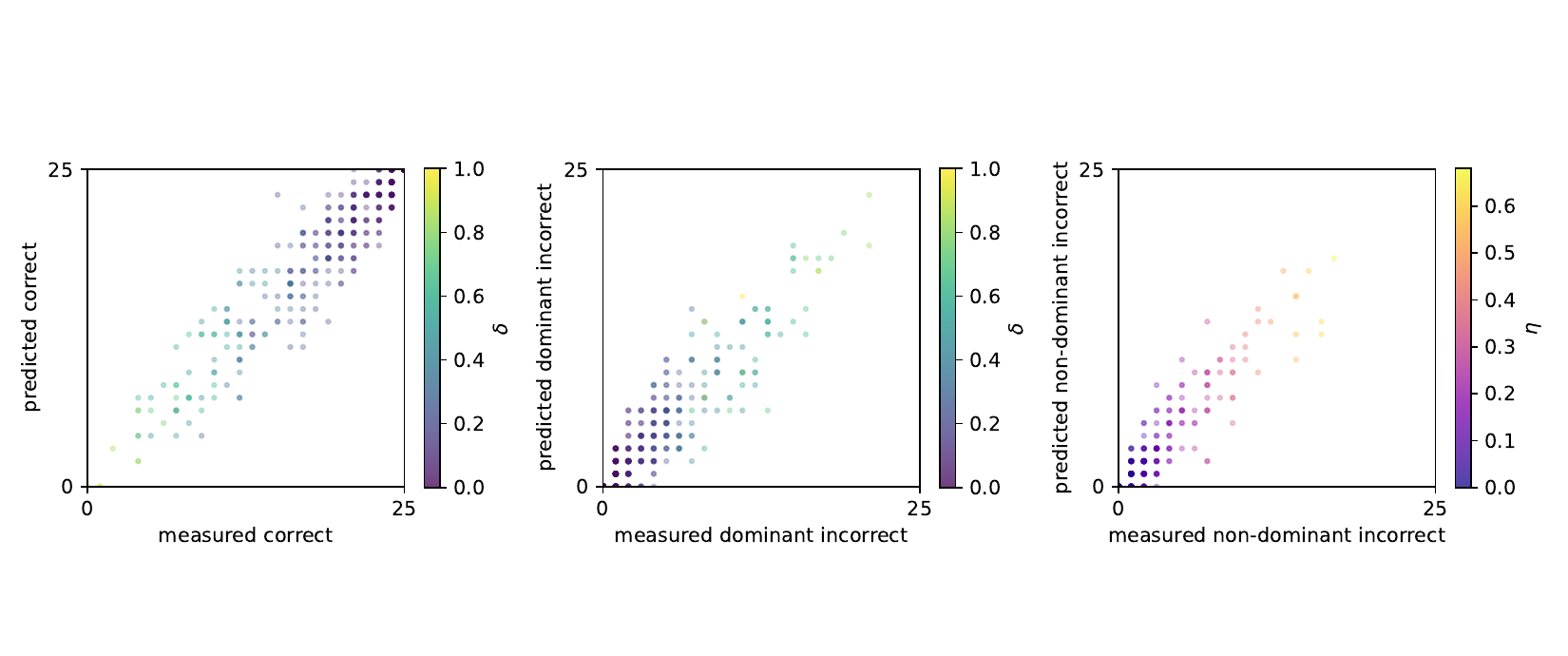}
        \caption{}
    \end{subfigure}
    \hfill
    \begin{subfigure}[h]{0.49\textwidth}
        \includegraphics[width=1\linewidth]{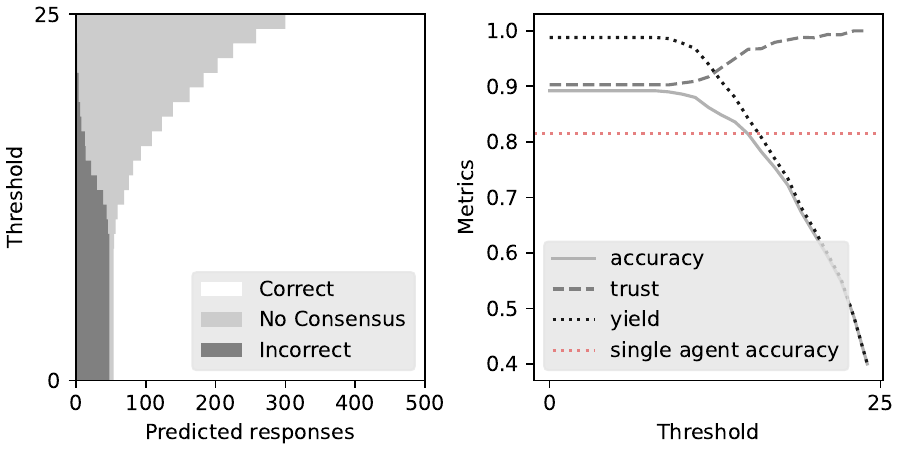}
        \caption{}
    \end{subfigure}
    \begin{subfigure}[h]{0.49\textwidth}
        \includegraphics[width=1\linewidth]{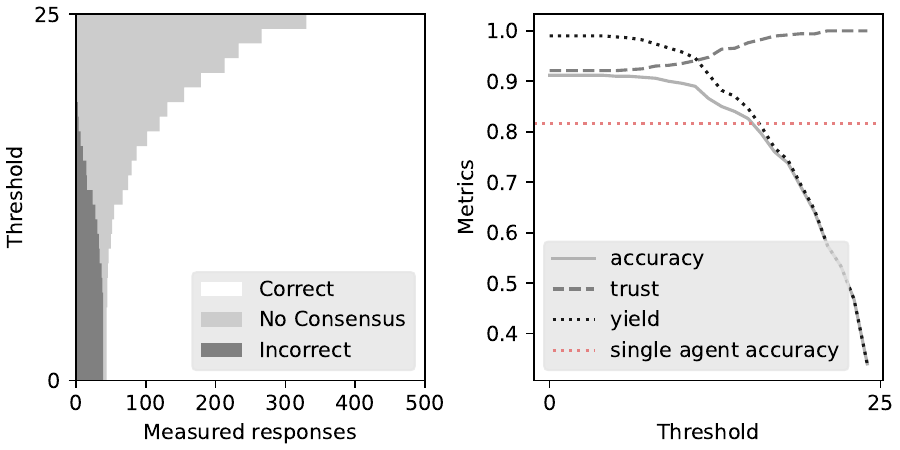}
        \caption{}
    \end{subfigure}
    \caption{Question parameters $\delta$ and $\eta$ estimated empirically over one arithmetic question domain (Figure \ref{fig:cot}b). (A) Simulated ensembles using the estimated $\delta$ and $\eta$ predict ensemble response distributions that are highly correlated with the measured distributions. Once evaluated, predicted metric performance (B) closely resembles that of the real ensemble (C).}
    \hfill
    \label{fig:empirical_delta_eta}
\end{figure}

\begin{figure}
    \centering
    \begin{subfigure}[h]{0.4\textwidth}
        \centering
        \includegraphics[width=1\linewidth]{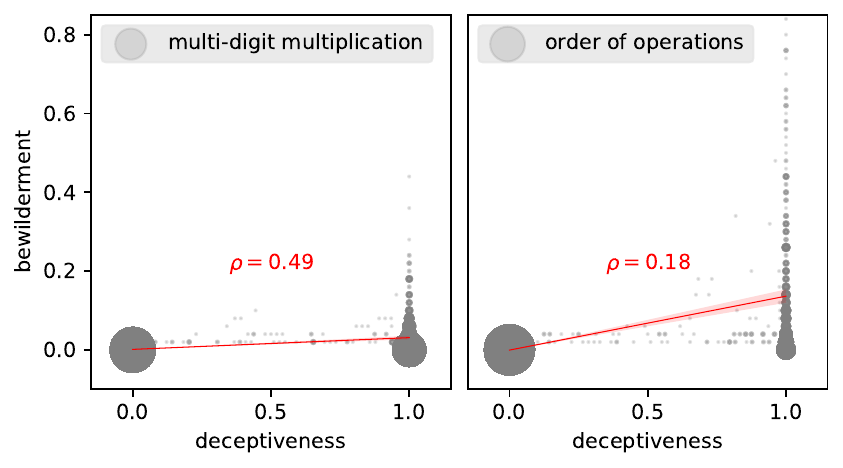}
    \end{subfigure}
    \hfill
    \begin{subfigure}[h]{0.58\textwidth}
        \centering
        \includegraphics[width=1\linewidth]{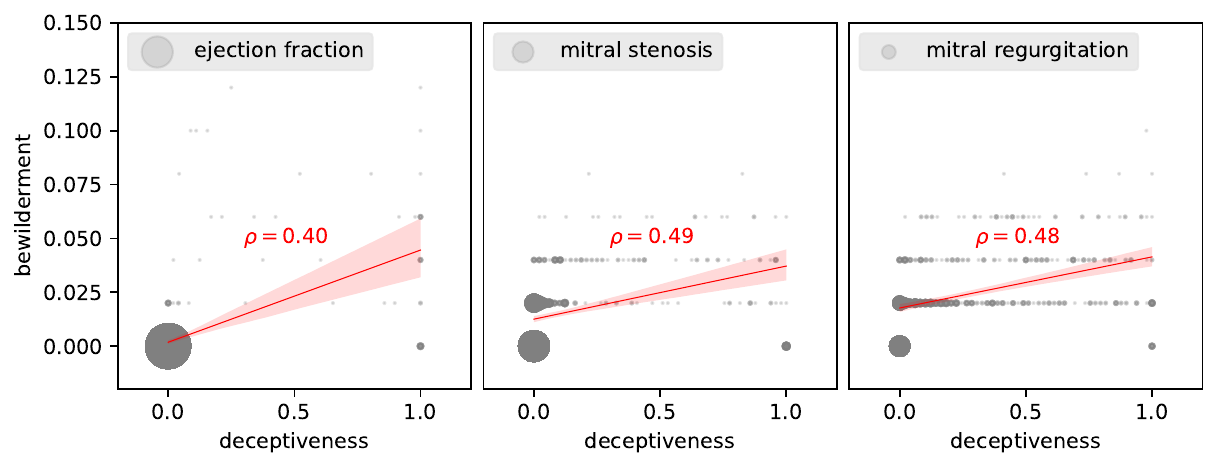}
    \end{subfigure}
    \caption{Two arithmetic problem domains and three clinical problem domains characterized by their empirically measured bewilderment and deceptiveness.}
    \label{fig:ensemble_delta_eta}
\end{figure}

\begin{figure}
    \centering
    \begin{subfigure}[h]{0.48\textwidth}
        \centering
        \begin{tikzpicture}
        \begin{axis}
        [
            title={Difficulty contour plot},
            view={0}{90},
            domain=0:1,
            colormap/viridis,
            unit vector ratio*=1 1 1 
            xlabel={\(\eta\)},
            ylabel={$\delta$},
            xmin=0,
            xmax=1,
            ymin=0,
            ymax=1,
            clip=false,
            scale = 0.95
        ]
        \addplot3[
            contour gnuplot={levels={0.9, 0.8, 0.7, 0.6, 0.5, 0.4, 0.3, 0.2, 0.1}},
            domain=0:1
        ]
        {x+y-x*y};
        \pgfplotsset{
            every axis/.append style={
                extra description/.code={
                    \node at (0.5,-0.15) {$\eta$};
                },
            },
        } 
        \end{axis}
        \end{tikzpicture}
    \end{subfigure}
    \hfill
    \begin{subfigure}[h]{0.48\textwidth}
        \centering
        \begin{tikzpicture}
            \begin{axis}[
                title=Difficulty plot,
                colormap/viridis,
                unit vector ratio=1 1 1,
                scale=1.4,
                xmin=0,
                xmax=1,
                ymin=0,
                ymax=1,
                zmin=0,
                zmax=1,
                xlabel=$\eta$,
                ylabel={$\delta$},
                zlabel={$d$}
            ]
            \addplot3[
                mesh,
                samples=20,
                domain=0:1,
            ]
            {x+y-x*y};
            \addlegendentry{\(d = \eta+\delta-\eta\delta\)}
            \end{axis}
            \end{tikzpicture}
    \end{subfigure}
    \caption{All questions exist on a manifold of difficulty, deceptiveness, and bewilderment within the unit cube.}
    \label{fig:ensemble_unit_cube}
\end{figure}
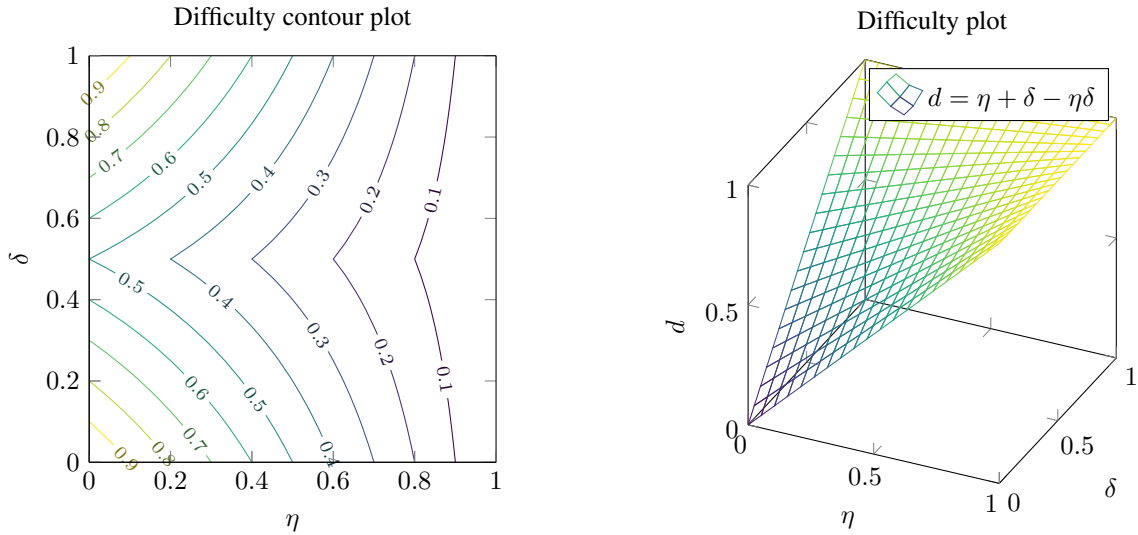

\begin{figure}
    \centering
    \begin{subfigure}[h]{0.48\textwidth}
        \centering
        \begin{tikzpicture}
        \begin{axis}
        [
            title={Maximum frequency contour plot},
            view={0}{90},
            domain=0:1,
            colormap/viridis,
            unit vector ratio*=1 1 1 
            xlabel={\(\eta\)},
            ylabel={$\delta$},
            xmin=0,
            xmax=1,
            ymin=0,
            ymax=1,
            clip=false,
            scale = 0.95
        ]
        \addplot3[
            contour gnuplot={levels={0.9, 0.8, 0.7, 0.6, 0.5, 0.4, 0.3, 0.2, 0.1}},
            domain=0:1
        ]
        {(1-x)*max(y,1-y)};
        \pgfplotsset{
            every axis/.append style={
                extra description/.code={
                    \node at (0.5,-0.15) {$\eta$};
                },
            },
        } 
        \end{axis}
        \end{tikzpicture}
    \end{subfigure}
    \hfill
    \begin{subfigure}[h]{0.48\textwidth}
        \centering
        \begin{tikzpicture}
            \begin{axis}[
                title=Maximum frequency plot,
                colormap/viridis,
                unit vector ratio=1 1 1,
                scale=1.4,
                xmin=0,
                xmax=1,
                ymin=0,
                ymax=1,
                zmin=0,
                zmax=1,
                xlabel=$\eta$,
                ylabel={$\delta$},
                zlabel={$f_{max}$}
            ]
            \addplot3[
                mesh,
                samples=20,
                domain=0:1,
            ]
            {(1-x)*max(y,1-y)};
            \addlegendentry{\(f_{max} = (1-\eta)max(\delta, 1-\delta\)}
            \end{axis}
            \end{tikzpicture}
    \end{subfigure}
    \caption{The expected maximum frequency of answers to a given question, combined with the voting threshold of the ensemble, determine the probability of reaching a consensus.} 
    \label{fig:ensemble_fmax_cube}
\end{figure}
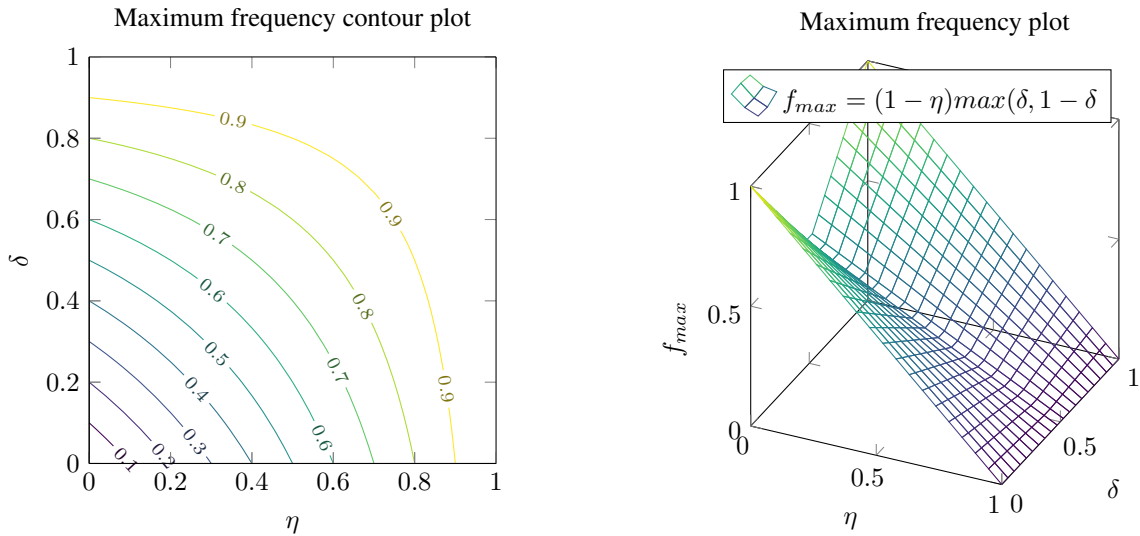

\end{document}